\documentclass{llncs}
\usepackage{times}
\usepackage{epic}
\usepackage{graphicx}
\usepackage{latexsym}
\usepackage{amsmath}

\usepackage{verbatim}
\usepackage{xspace}
\usepackage{rotating}

\newcommand{\myomit}[1]{}

\newcommand{\constraint}[1]{\mbox{\sc #1}}
\newcommand{\regular}{\constraint{Regular}\xspace}
\newcommand{\CFG}{\constraint{Grammar}\xspace}

\newcommand{\edit}{\mbox{$\mbox{\sc EditDistance}$}\xspace}

\newcommand{\Xbf}{\mbox{{\bf X}}}
\newcommand{\Ybf}{\mbox{{\bf Y}}}
\newcommand{\Zbf}{\mbox{{\bf Z}}}
\newcommand{\calL}{\mbox{\ensuremath{\mathcal L}}}

\newcommand{\WCFG}{\constraint{WeightedCFG}\xspace}

\newcommand{\gives}{\rightarrow}

\newcommand{\calA}{\ensuremath{\mathcal{A}}}

\begin{document}



\input epsf





\newcommand{\myset}[1]{\ensuremath{\mathcal #1}}

\renewcommand{\theenumii}{\alph{enumii}}
\renewcommand{\theenumiii}{\roman{enumiii}}
\newcommand{\figref}[1]{Figure \ref{#1}}
\newcommand{\tref}[1]{Table \ref{#1}}
\newcommand{\myldots}{\ldots}

\newtheorem{mydefinition}{Definition}
\newtheorem{mytheorem}{Theorem}
\newtheorem{myexample}{Example}
\newtheorem{mytheorem1}{Theorem}
\newcommand{\myproof}{\noindent {\bf Proof:\ \ }}
\newcommand{\myqed}{$\clubsuit$}
\newcommand{\simplify}{\mbox{\ensuremath{\mathit{simplify}}}}
\newcommand{\unfold}{\mbox{\ensuremath{\mathit{unfold}}}}
\newcommand{\myOmit}[1]{}
\newcommand{\tuple}[1]{\mbox{\ensuremath{\left\langle #1 \right\rangle}}}
\newcommand{\set}[1]{\mbox{\ensuremath{\left\{ #1 \right\}}}}

\title{Restricted Global Grammar Constraints.\thanks{NICTA is funded by 
the Australian Government's Department of Broadband, 
Communications,  and the Digital Economy and the 
Australian Research Council. 
}}

\author{George Katsirelos\inst{1} \and
Sebastian Maneth\inst{2}
\and
Nina Narodytska\inst{2}
\and
Toby Walsh\inst{2}}
\institute{NICTA, Sydney, Australia,
email: george.katsirelos@nicta.com.au
\and NICTA and University of NSW,
Sydney, Australia, email: sebastian.maneth@nicta.com.au, ninan@cse.unsw.edu.au, toby.walsh@nicta.com.au}

\date{27 April 2009}

\maketitle
\begin{abstract}
We investigate the global $\CFG$ constraint
over restricted classes of context free grammars like deterministic
and unambiguous context-free grammars.  We show that detecting disentailment
for the $\CFG$ constraint in these cases is as hard as parsing an unrestricted 
context free grammar. We also consider the class of linear grammars and give
a propagator 
that runs in quadratic time. 
Finally, to demonstrate the use of linear
grammars, we show that a weighted linear $\CFG$ constraint can 
efficiently encode the \edit constraint,
and a conjunction of the \edit constraint 
and the \regular constraint.

\end{abstract}
\sloppy

\section{Introduction}

In domains like staff scheduling, regulations can 
often be naturally expressed using 
formal grammars. Pesant~\cite{pesant1} introduced 
the global $\regular$ constraint 
to express problems 
using finite automaton. Sellmann~\cite{grammar2} 
and Quimper and Walsh~\cite{qwcp06} 
then introduced the global $\CFG$ constraint for context-free grammars.
Unlike parsing which only finds a single
solution, propagating such a constraint essentially
considers all solutions. Nevertheless, a propagator for the 
$\regular$ constraint runs in linear time and a propagator for the  $\CFG$ 
constraint runs in cubic time just like the corresponding parsers. 
Subsequently, there has been research
on more
efficient propagators for these
global constraints~\cite{qwcp07,QuimperAAAI08,KS08,lagerkvistthesis,gkatsiCPAIOR09}.
Whilst research has focused on regular and
unrestricted context-free languages,
a large body of work in formal language theory
considers 
grammars between regular and context-free. 
Several restricted forms of context-free grammars have been
proposed that permit linear parsing algorithms
whilst being more expressive than regular grammars. 
Examples of such grammars are LL(k), LR(1), and LALR. 
Such grammars
play an important role in compiler theory. 
For instance, yacc generates
parsers that accept LALR languages. 
%

In this paper we explore the gap between the
second and third levels of the Chomsky hierarchy
for classes of grammars which can be propagated more efficiently
than context-free grammars.
These classes of grammar are attractive because
either
they have a linear or quadratic time membership test 
(e.g., LR(k) and linear grammars, respectively)
or they permit counting of strings of given length
in polynomial time 
(e.g., unambiguous grammars). The latter
may be useful in branching heuristics.
One of our main contributions
is a lower bound on the time complexity
for propagating grammar constraints using deterministic or
unambiguous grammars. We prove that detecting
disentailment for such constraints has the same time complexity as
the best parsing algorithm for an arbitrary context-free grammar.
Using LL(k) languages or unambiguous grammars 
does not therefore improve the efficiency of propagation.  
Another contribution
is to show that linearity of the grammar
permits quadratic time propagation. 
We show that we can encode
an \edit constraint and a combination of an \edit constraint
and a \regular constraint using such a linear grammar. 
Experimental results show that   
this encoding is
very efficient in practice.

\section{Background}

A context-free grammar is a tuple $G=\tuple{N, T, P, S}$, 
where 
$N$ is a finite set of \emph{non-terminal} symbols,
$T$ is a finite set of \emph{terminal} symbols,
$P$ is a set of \emph{productions}, 
and $S\in N$ is the \emph{start symbol}. 
A production is of the form $A \gives \alpha$
where $A\in N$ and $\alpha\in(N\cup T)^+$.
The derivation relation $\Rightarrow_G$ induced by $G$ is
defined as follows:
for any $u,v\in(N\cup T)^*$, 
$uAv\Rightarrow_G u\alpha v$ 
if there exists a production $A\to\alpha$ in $P$.
Sometimes we additionally index $\Rightarrow_G$ by the
production that is applied.
The transitive, reflexive closure of $\Rightarrow_G$ is
denoted by $\Rightarrow_G^*$.
A string in $s\in T^*$ is \emph{generated} 
by $G$ if $S\Rightarrow_G^* s$.
The set of all strings generated by $G$ is 
denoted $L(G)$.
Note that 
this does not
allow the \emph{empty string} $\varepsilon$ as 
right-hand side of a production. Hence, $\varepsilon\not\in L(G)$.
This is not a restriction: we can add
a new start symbol $Z$ with productions 
$Z\to \varepsilon \mid S$ to our grammars. Our results 
can easily be generalized to such $\varepsilon$-enriched grammars.
We denote the length of a string $s$ by $|s|$.
The \emph{size} of $G$, denoted by $|G|$, is 
$\sum_{A\to\alpha\in P} |A\alpha|$.

A context-free grammar is in \emph{Chomsky form} if all productions
are of the form $A \gives BC$ where $B$ and $C$ are non-terminals or
$A \gives a$ where $a$ is a terminal.  Any $\varepsilon$-free
context-free grammar $G$ can be
converted to an equivalent grammar $G'$ 
in Chomsky form with at most a linear increase in its size;
in fact, $|G'|\leq 3|G|$, see Section~4.5 in~\cite{hopull79}.
A context-free grammar is in \emph{Greibach form} if all productions
are of the form $A \gives a\alpha$ where $a$ is a terminal and $\alpha$ is 
a (possibly empty) 
sequence of non-terminals. 
Any context-free grammar $G$ can be converted to an equivalent 
grammar $G'$ in Greibach form with at most a polynomial increase in its
size; in fact, the size of $G'$ is in $O(|G|^4)$ in general, and is in
$O(|G|^3)$ if $G$ has no chain productions of the form $A\to B$ for
nonterminals $A,B$, see~\cite{blukoc99}.
A context-free grammar is \emph{regular} if all productions are
of the forms $A\to w$ or $A\to wB$ for non-terminals $A,B$ and $w\in T^+$.


%
%
%
%
%

\section{Simple Context-Free Grammars}

In this section we show that propagating a
{\em simple} context-free grammar constraint is at least as 
hard as parsing an (unrestricted) context-free
grammar. 
A grammar $G$ is \emph{simple} if it is in Greibach form, and
for every non-terminal $A$ and terminal $a$ there is at most
one production of the form $A\to a\alpha$.
Hence, restricting ourselves
to languages recognized by simple context-free grammars 
does not improve 
the complexity of propagating a global grammar
constraint. 
Simple context-free languages are included in the 
deterministic context-free languages 
(characterized by deterministic push-down automata), and also 
in the $LL(1)$ languages~\cite{flhandbook}, so this result
also holds for propagating these classes of languages.
Given finite sets $D_1,\dots,D_n$, their 
\emph{Cartesian product language}  
$L(R_{D_1 , \ldots , D_n})$
is the cross product of the domains 
$\{a_1a_2\cdots a_n\mid a_1\in D_1,\dots,a_n\in D_n\}$.
Following ~\cite{grammar2}, we define the global $\CFG$ constraint:
\begin{definition}
\label{p:grammar_constraint}
The $\CFG([X_1,\ldots,X_n], G)$ constraint is true for an assignment variables $X$ iff
a string formed by this assignment belongs to $L(G)$.
\end{definition}

From Definition~\ref{p:grammar_constraint}, we observe that finding a support for 
the grammar constraint is equivalent to intersecting the context-free language
with the Cartesian product language of the domains. 

\begin{proposition}
\label{p:cfg_grammar}
Let $G$ be a context-free grammar, $X_1,\ldots,X_n$ be a sequence of variables with domains 
$D(X_1), \ldots, D(X_n)$. Then
$L(G) \cap L(R_{D(X_1),\dots, D(X_n)}) \neq 0$ iff
$\CFG([X_1,\ldots,X_n], G)$ has a support.
\end{proposition}

Context-free grammars are effectively
closed under intersection with regular grammars.
To see this, consider a context-free grammar $G$ in Chomsky 
form and a
regular grammar $R$.  
Following the ``triple construction'', the intersection
grammar has non-terminals of the form $\tuple{F,A,F'}$ where 
$F,F'$ are non-terminals of $R$ and $A$ is a non-terminal
of $G$. Intuitively, $\tuple{F,A,F'}$ generates strings 
$w$ that are generated by $A$ and also by $F$, through a derivation 
from $F$ to $wF'$.
If $A\to BC$ is a production of $G$, then we
add, for all non-terminals $F,F',F''$ of $R$, the production
$\tuple{F,A,F''}\to\tuple{F,B,F'}\tuple{F',C,F''}$.
The resulting grammar is
$O(|G|n^3)$ in size where $n$ is the number of non-terminals of $R$.
This is similar to the construction of Theorem~6.5 in~\cite{hopull79}
which uses push-down automata instead of grammars.
Since emptiness of context-free grammars takes
linear time (cf.~\cite{hopull79}) 
we obtain through Proposition~\ref{p:cfg_grammar}
a cubic time algorithm to check whether a global constraint 
$\CFG([X_1,\ldots,X_n], G)$ has support.
In fact, this shows that we can efficiently propagate 
more complex constraints, such as the conjunction of a context-free
with a regular constraint. Note that if $R$ is a Cartesian product language
then the triple construction generates the same result as
the CYK based propagator for the \CFG constraint~\cite{grammar2,qwcp06}.

We now show that for 
\emph{simple} context-free grammars $G$, detecting
disentailment of the constraint 
$\CFG([X_1,\ldots,X_n], G)$, i.e. testing whether it has a solution, is at least as hard as
parsing an arbitrary context-free grammar.



\begin{theorem}
\label{t:grammar_hard}
Let $G$ be a context-free grammar in Greibach form 
and $s$ a string of length $n$.
One can construct in $O(|G|)$ time a simple context-free grammar
$G'$ and in $O(|G|n)$ time
a Cartesian product language $L(R_{D(X_1),\dots, D(X_n)})$ 
such that $L(G')\cap L(R_{D(X_1),\dots, D(X_n)})\neq\emptyset$ iff $s\in L(G)$. 
\end{theorem}
\begin{proof}
The idea behind the proof is to determinize an unrestricted context free
grammar $G$ by mapping each terminal in $G$ to a set of pairs -- the terminal and a 
production that can consume this terminal. This allows us to carry information about 
the derivation inside a string in $G'$. Then, we construct a Cartesian product 
language $L(R_{D(X_1),\dots, D(X_n)})$ over these pairs so that
all strings from this language map only to the string $s$.  
Let $G=\tuple{N,T,P,S}$ and 
fix an arbitrary order of the productions in $P$.
We now construct the grammar $G'=\tuple{N,T',P',S}$.
For every $1\leq j\leq |P|$, if the $j$-th production
of $P$ is $A\to a\alpha$ then let $(a,j)$ be a new symbol
in $T'$ and let the production $A\to(a,j)\alpha$ be in $P'$.
Next, we construct the  Cartesian product language.
We define $D(X_i) = \{(a,j)| (s_i = a) \wedge (a,j) \in T'\}$, $i=1,\ldots,n$ and $s_i$
is the $i$-th letter of $s$.
Clearly, $G'$ is constructed in $O(|G|)$ time and
$L(R_{D(X_1),\dots, D(X_n)})$ in $O(|P|n)$ time.

 ($\Rightarrow$) Let $L(G')\cap L(R_{D(X_1),\dots, D(X_n)})$ be non empty. 
Then there exits a string $s'$ that belongs to the intersection.
Let $s'=(a_1,i_1)\cdots (a_n,i_n)$.
By the definition of $L(R_{D(X_1),\dots, D(X_n)})$, the string $a_1a_2\cdots a_n$ must 
equal $s$. Since $s'\in L(G')$, there must be a derivation by $G$ of 
the form
\[
S\Rightarrow_{G,p_1}a_1\alpha\Rightarrow_{G,p_2}a_1a_2\alpha'\dots
\Rightarrow_{G,p_n} a_1\cdots a_n
\]
where $p_j$ is the $j$-th production in $P$.
Hence, $s\in L(G)$.

($\Leftarrow$) Let $s\in L(G)$. Consider a derivation sequence of
the string $s$. We replace every symbol $a$ in $s$ that was derived
by the $i$-th production of $G$ by $(a,i)$. 
By the construction of $G'$, the string $s'$ is in
$L(G')$. Moreover, $s'$ is also in $L(R_{D(X_1),\dots, D(X_n)})$.
  \qed
\end{proof}

Note that context-free parsing has a quadratic time lower bound,
due to its connection to matrix multiplication~\cite{lee02}.
Given this lower bound and the fact that the
construction of Theorem~\ref{t:grammar_hard} requires only linear time, 
we can deduce the following.

\begin{corollary}
\label{p:grammar_dc_hardness}
Let $G$ be a context-free grammar. If $G$ is simple
(or deterministic or $LL(1)$) then detecting disentailment of
$\CFG([X_1,\ldots,X_n], G)$ is
at least as hard as context-free parsing of a string of
length $n$.
\end{corollary}
\myOmit{
\begin{proof}
Here, ``at least as hard'' means, with respect to 
reductions that are polynomial in the size of the grammar
and linear in $n$. We transform a given context-free 
grammar $H$ into Greibach form. If $H$ is in Chomsky
form (as for CYK) then this takes $O(|H|^3)$ time
(otherwise $O(|H|^4)$)
as mentioned in the Background section. According to
Theorem~\ref{t:grammar_hard} we construct a simple grammar $H'$
and define $D(X_i)$ as the letter produced by nonterminal
$A_i$. With Proposition~\ref{p:cfg_grammar}, $s\in L(G)$
iff $\CFG([X_1,\ldots,X_n], G)$ has support.
\qed
\end{proof}
}

We now show the converse to Theorem~\ref{t:grammar_hard}
which reduces intersection emptiness of a context-free
with a regular grammar, to the membership problem
of context-free languages. This
shows that the time complexity of detecting
disentailment for the $\CFG$ constraint is the same
as the time complexity of the best parsing
algorithm for an arbitrary context free grammar. Therefore,
our result
shows that detecting disentailment takes $O(n^{2.4})$ time 
\cite{copwin90},
as in the best known algorithm for Boolean matrix multiplication. 
It does not, however,
improve the asymptotic complexity of a domain consistency 
propagator for the $\CFG$ constraint~\cite{grammar2,qwcp06}.

\begin{theorem}
\label{t:grammar_complete}
Let $G=\langle N,T,P,S \rangle$ be a context-free grammar 
and $L(R_{D(X_1),\dots, D(X_n)})$ be  Cartesian product language.
One can construct in time $O(|G|+|T|^2)$ 
a context-free grammar $G'$ and in time
$O(n|T|)$ a string $s$ such that 
$s\in L(G')$ iff $L(G)\cap L(R_{D(X_1),\dots, D(X_n)})\neq\emptyset$. 
\end{theorem}
\begin{proof} (Sketch)
We assign an index to each terminal in $T$. 
For each position $i$ of the strings of $R$, we create a bitmap of
the alphabet that describes the terminals that may appear in that
position. The $j$-th bit of the bitmap 
is $1$ iff the symbol with index $j$ may appear at position $i$.
The string $s$ is the concatenation 
of the bitmaps for each position and has size $n|T|$.
First, we add $B\to 0$ and $B\to 1$ to $G'$.
For each terminal in $T$ with index $j$, 
we introduce 
$T_j\to B^{j-1}1B^{|T|-j}$
into $G'$ to accept any bitmap with 1 at the
$j$-th position. 
Then, for each production in $G$ of the form $A\to a\alpha$  
such that the index of $a$ is $j$, we add
$A\to T_j\alpha$  to $G'$.
In this construction, every production in $G'$ except for those with
$T_i$ on the left hand side can be uniquely mapped to a production in $G$. 
It can be shown that $s\in L(G')$ iff $L(G)\cap L(R_{D(X_1),\dots, D(X_n)}) \neq\emptyset$. 
\myomit{
($\Rightarrow$) Suppose $s' \in L(G')$. Consider a derivation sequence
of $s'$. We construct a string $s \in L(G)\cap L(R)$. Note that
every string that belongs to $L(G')$ has length $k|T|$ for some $k$.
Moreover, it is partitioned in blocks of $|T|$ symbols and every block
of is generated by one of the non-terminals $T_p$. We construct $s$ by
placing at the $j$-th position the symbol with index $i$ if the
$j$-th block of symbols of $s'$ was generated by the non-terminal
$T_i$. Clearly this belongs to $R$. 
Then, a derivation of $s$ in $G$ can be created by
removing from the derivation of $s'$ productions with $T_i$ on the
left hand side and replacing the rest with the corresponding
production in $G$, so $s\in L(G)\cap L(R)$ and  
$L(G)\cap L(R)\neq \emptyset$.

($\Leftarrow$) Suppose $L(G) \cap L(R)\neq\emptyset$ and let $s$
be a string in the intersection. We construct a string $s'$ that
belongs to $L(G')$ by replacing the symbol with index $i$ 
at position $j$ of $s$ with a block of $|T|$ symbols with 1 at the
$i$-th position and 0 elsewhere. 
We create a derivation of $s'$ in $G'$ by replacing each production 
$A\gives a \alpha$ in it with the pair 
of productions $A \gives T_i \alpha$, $T_i \gives (0|1)^{i-1}1(0|1)^{|T|-i}$
where $i$ is the index of $a$. Thus, $s \in L(G')$.}
\qed
\end{proof}

\section{Linear Context-Free Grammars}

A context-free grammar is \emph{linear} if every production contains
at most one non-terminal in its right-hand side.
The linear languages are a proper superset of the regular languages
and are a strict subset of the context-free languages.
Linear context-free grammars possess two important properties:
(1) membership of a given string of length $n$ can be 
checked in time $O(n^2)$ (see Theorem~12.3 in~\cite{wagwec86}), and
(2) the class is closed under intersection with
regular grammars (to see this, apply the
``triple construction'' as explained
after Proposition~\ref{p:cfg_grammar}).
The second property opens the possibility of constructing a polynomial
time propagator for a conjunction of the the linear $\CFG$
and the $\regular$ constraints. 
 Interestingly, we can show that a CYK-based propagator for this
type of grammars runs in quadratic time. This is 
then the third example of a grammar, besides regular and context-free
grammars, where the asymptotic time complexity 
of the parsing algorithm and that of the corresponding 
propagator are equal. 

\begin{theorem}
\label{t:andor_linear}
Let $G$ be a linear grammar and $\CFG([X_1,\ldots,X_n],G)$
be the corresponding global constraint.  
There exists a domain consistency propagator for this constraint that
runs in $O(n^2|G|)$ time. 
\end{theorem}
\begin{proof}
We convert $G=\tuple{N,T,P,S}$ into CNF. 
Every linear grammar can be converted into the form
$A \gives aB$, $A\gives Ba$ and $A\gives a$,  where $a,b \in T$ and  $A,B\in N$
(see Theorem~12.3 of~\cite{wagwec86}) in $O(|G|)$ time. To obtain CNF
we replace every terminal $a\in T$ that occurs in a production on
the right hand side with a new non-terminal $Y_a$ and introduce
a production $Y_a \gives a$.

Consider the CYK-based domain consistency propagator for an arbitrary context-free grammar
constraint~\cite{qwcp06,grammar2}. The algorithm runs in two stages. In the first stage,
it constructs in a bottom-up fashion
a dynamic programing table $V_{n \times n}$,
where an element $A$ of $V_{i,j}$ is a potential non-terminal
that generates a substring from the domains of variables $[X_{i},\ldots,X_{i+j}]$.
In the second stage, it
performs a top-down traversal of $V$ 
and marks an element $A$ of  $V_{i,j}$ iff
it is reachable from the starting non-terminal $S$
using productions of the grammar and elements
of $V$.
It then removes unmarked elements, including terminals.
If it removes a terminal at column $i$ of the table, it prunes 
the corresponding value
of variable $X_i$.
%

The complexity of this algorithm is bounded by the number of possible 1-step derivations
from each non-terminal in the table.  
Let $G'=\tuple{N',T',P',S'}$ be an arbitrary 
context free grammar.
There are $O(|N'|n^2)$ non-terminals in
the table and each non-terminal can be expanded in $O(F'(A)n)$
possible ways, where $F'(A)$ is the number of 
productions in $G'$ with non-terminal $A$ on the left-hand side.
Therefore, the total time complexity of the propagator 
for unrestricted context-free grammars is 
$n^2\sum_{A \in N'} {nF'(A)} = O(n^3|G'|)$.
In contrast, 
the number of possible 1-step derivations 
from each non-terminal
in linear grammars 
is bounded by $O(F(A))$.
Therefore, the propagator runs in $O(n^2|G|)$ for a linear grammar $G$.
\qed
\end{proof}

Theorem~\ref{t:andor_linear} can be extended to the 
weighted form of the linear $\CFG$ constraint, $\WCFG$~\cite{katnarwal08}. 
A weighted grammar is annotated with a weight for each production
and the weight of a derivation is the sum of all weights used in it.
The linear $\WCFG(G,Z,[X_1,\ldots,X_n])$ constraint
holds iff an assignment $X$ forms a string belonging to 
the weighted linear grammar $G$ and 
the minimal weight derivation of $X$ is less than or equal to $Z$. 
The domain consistency propagator for the $\WCFG$ constraint 
is an extension of the propagator for $\CFG$
that computes additional information for each non-terminal $A \in V_{i,j}$---the 
minimum and the maximum weight derivations from $A$. 
Therefore, this algorithm has the same time and space asymptotic complexity 
as the propagator for $\CFG$, so
the complexity analysis for the linear $\WCFG$ 
constraint 
is identical to the non-weighted case.   

It is possible to restrict linear grammars further, so that
the resulting global constraint problem is solvable in 
\emph{linear time}. As an example, consider
``fixed-growth'' grammars in which there exists
$l$ and $r$ with $l+r\geq 1$ such that every production 
is of the form either $A\to w\in T^+$ or 
$A\to uBw$ where the length of $u\in T^*$ equals $l$ and
the length of $w\in T^*$ equals $r$. 
In this case, the triple construction 
(explained below Proposition~\ref{p:cfg_grammar}) generates
$O(|G|n)$ new non-terminals implying linear time
propagation (similarly, CYK runs in linear time 
as it only generates non-terminals on 
the diagonal of the dynamic program).
A special case of fixed-growth grammars are
regular grammars which have $l=1$ and $r=0$ (or vice versa).

\section{The \edit Constraint}

To illustrate linear context-free grammars,
we show how to encode an edit distance constraint into such
a grammar. 
$\edit([X_1,\ldots,X_n, Y_1,\ldots,Y_m], N)$
holds iff the edit distance between assignments of 
two sequences of variables $\Xbf$ and $\Ybf$ is less than or equal to $N$. 
The edit distance is the minimum number of deletion,
insertion and substitution operations required
to convert one string into another. Each of these operations 
can change one symbol in a string. W.L.O.G. we assume that $n = m$. 
We will show that the $\edit$  constraint can be encoded 
as a linear $\WCFG$ constraint. The idea of the encoding 
is to parse matching substrings using productions of weight 0
and to parse edits
using productions of weight 1.

We convert $\edit([\Xbf, \Ybf], N)$ 
into a linear 
$\WCFG([\Zbf_{2n+1}, N, G_{ed})$ constraint. 
The first $n$ variables in the sequence $\Zbf$ are equal to the 
sequence $\Xbf$, 
the variable $Z_{n+1}$ is ground to the sentinel symbol $\#$ 
so that the grammar can 
distinguish the sequences $\Xbf$ and $\Ybf$, and 
the last $n$ variables of the sequence $\Zbf$
are equal to the reverse of the sequence $\Ybf$.
We define the linear weighted grammar $G_{ed}$ as follows.
Rules $S \gives d S d$ with weight $w = 0$, $\forall d \in D(X) \cup D(Y)$,
capture matching terminals, rules $S \gives  d_1 S d_2$ with $w=1$, $\forall d_1 \in D(X), d_2 \in D(Y), d_1 \neq d_2$,
capture replacement, rules $S \gives  d S|S d$ with  $w=1$, $\forall d \in D(X)$,
capture insertions and deletions. Finally, the rule $S \gives  \#$ with weight $w = 0$
generates the sentinel symbol.
%
%
%
%
%
As discussed in the previous section,
the propagator for the linear $\WCFG$ constraint
takes $O(n^2|G|)$ time.
%
Down a branch of the search tree, the time complexity is
$O(n^2|G|ub(N))$.

We can use this encoding of
the $\edit$ constraint into a linear $\WCFG$
constraint to construct propagators for
more complex constraints. For instance,
we can exploit the fact that linear
grammars are closed under intersection
with regular grammars to propagate
efficiently the conjunction of an
\edit constraint and \regular constraints
on each of the sequences $\Xbf$,$\Ybf$.
More formally, let $\Xbf$ and $\Ybf$ be two sequences of variables of
length $n$ subject to the constraints $\regular(\Xbf, R_1)$,
$\regular(\Ybf, R_2)$ and $\edit(\Xbf, \Ybf, N)$.
We construct a domain consistency propagator for the conjunction of these
three constraints, by computing a grammar that generates 
strings of length $2n+1$ which satisfy the conjunction.
First, we construct an automaton
that accepts $ \calL(R_1)\#\calL(R_2)^R$. 
%
This language is regular and requires an automaton 
of size $O(|R_1|+|R_2|)$.
Second, we intersect this with
the linear weighted grammar that encodes the 
$\edit$ constraint using the ``triple construction''.
The size of the obtained grammar is $G_{\wedge} =
|G_{ed}|(|R_1|+|R_2|)^2$ and this grammar is a weighted linear
grammar. Therefore, we can use the linear $\WCFG(\Zbf, N, G_{\wedge})$
constraint to encode the conjunction.
Note that the size of $G_{\wedge}$ is only quadratic in $|R_1|+|R_2|$,
because $G_{ed}$ is a linear grammar. The time complexity to
enforce domain consistency on this conjunction of constraints is
$O(n^2|G_{\wedge}|) = O(n^2d^2(|R_1|+|R_2|)^2)$ for each invocation
and $O(n^2d^2(|R_1|+|R_2|)^2ub(N))$ down a branch of the search tree.
\begin{table}[htb]
\begin{center}
{\small
\caption{\label{t:t1} 
Performance of the encoding
into $\WCFG$ constraints shown in: 
number of instances solved in 60 sec / average number of choice points 
/ average time to solve.
}
\begin{tabular}{| rrr|ccc|ccc|}
\hline
&
\multicolumn {2}{r|} {$n$ ~~ $N$}
&\multicolumn {3}{|c|}{$ED_{Dec}$}
&\multicolumn {3}{|c|}{$ED_{\wedge}$} \\
\hline 
 & 
 &
 &
 \#solved &
 \#choice points&
 time&
 \#solved &
 \#choice points&
 time\\
\hline 
&\multicolumn {2}{r|} {15~~~~2} & \textbf{100}&  \textbf{     29} & \textbf{  0.025} & \textbf{100}&      6 &  0.048 \\ 
&\multicolumn {2}{r|} {20~~~~2}  & \textbf{100}&    661 &  0.337 &  \textbf{100}&  \textbf{      6} & \textbf{  0.104} \\ 
&\multicolumn {2}{r|} {25~~~~3}  & 93 &   2892 &  2.013 & \textbf{100}&  \textbf{     10} & \textbf{  0.226} \\ 
&\multicolumn {2}{r|} {30~~~~3}  & 71 &   6001 &  4.987 & \textbf{100}&  \textbf{     12} & \textbf{  0.377} \\ 
&\multicolumn {2}{r|} {35~~~~4}  & 58 &   5654 &  6.300 & \textbf{100}&  \textbf{     17} & \textbf{  0.667} \\ 
&\multicolumn {2}{r|} {40~~~~4}  & 40 &   3140 &  4.690 & \textbf{100}&  \textbf{     17} & \textbf{  0.985} \\ 
&\multicolumn {2}{r|} {45~~~~5}  & 36 &   1040 &  2.313 & \textbf{100}&  \textbf{     19} & \textbf{  1.460} \\ 
&\multicolumn {2}{r|} {50~~~~5}  & 26 &   1180 &  4.848 & \textbf{100}&  \textbf{     24} & \textbf{  1.989} \\ 
\hline 
\hline 
\multicolumn {3}{|r|} { TOTALS }& & & & & &\\ 
\multicolumn {3}{|r|} {solved/total}& \multicolumn {3}{|c|} {  524 /800}& \multicolumn {3}{|c|} {\textbf{800} /800}\\ 
\multicolumn {3}{|r|} {avg time for solved}& \multicolumn {3}{|c|}{  2.557} & \multicolumn {3}{|c|}{\textbf{  0.732}} \\ 
\multicolumn {3}{|r|} {avg choice points for solved}& \multicolumn {3}{|c|}{   2454} &\multicolumn {3}{|c|}{\textbf{    14}} \\ 
\hline 
\end{tabular}}
\end{center}
\end{table}

To evaluate the performance of the $\WCFG(\Zbf, N, G_{\wedge})$ constraint 
we carried out a series  of experiments on random problems.
In out first model the conjunction of the \edit constraint and two \regular constraints
was encoded with a single $\WCFG(\Zbf, N, G_{\wedge})$ constraint. We call this model 
$ED_{\wedge}$. The second model contains the \edit constraint, encoded as $\WCFG(\Zbf, N, G_{ed})$, 
and two \regular constraints.  
The \regular constraint for the model $ED_{Dec}$ is implemented 
using a decomposition into ternary table constraints~\cite{qwcp06}. 
The $\WCFG$
constraint is implemented with an incremental
monolithic propagator~\cite{katnarwal08}.
The first $\regular$ constraint ensures that there are at most two consecutive
values one in the sequence. The second encodes
a randomly generated string of $0$s and $1$s. 
To make problems harder, we enforced the $\edit$ constraint
and the \regular constraints on  two sequences 
$\Xbf\#(\Ybf)^R$ and $\Xbf'\#(\Ybf')^R$ of the same length $2n+1$.
The \edit constraint and the first \regular constraint are identical for these
two sequences, while $\Ybf$ and $\Ybf'$ correspond to different randomly generated strings of 0s and 1s. 
Moreover, $\Xbf$ and $\Xbf'$ overlap on $15\%$ of randomly chosen variables. 
For each possible value of $n\in \{15,20,25,30,35,40,45,50\}$,
we generated 100 instances. Note that $n$ is the length of
each sequence $\Xbf$, $\Ybf$, $\Xbf'$ and $\Ybf'$. $N$ is the maximum edit
distance between $\Xbf$ and $\Ybf$ and  between $\Xbf'$ and $\Ybf'$.
We used a random value and variable ordering and a time out of $60$ sec.
Results for different values of $n$ are presented in Table~\ref{t:t1}.
As can be seen from the table, the 
model $ED_{\wedge}$ significantly outperforms the model  $ED_{Dec}$  for larger problems,
but it is slightly slower for smaller problems. 
Note that the model $ED_{\wedge}$ solves many more instances compared to  
$ED_{Dec}$. 

\vspace*{-5pt}
\section{Conclusions}

Unlike parsing,
restrictions on context free grammars such as determinism do not
improve the efficiency of propagation of the
corresponding global \CFG constraint.
On the other hand, one specific syntactic restriction,
that of linearity, allows propagation in 
quadratic time. We
demonstrated an application of such a restricted grammar
in encoding the \edit constraint and more complex constraints.

\vspace*{-5pt}
\bibliographystyle{abbrv}

\begin{thebibliography}{10}

\bibitem{blukoc99}
N.~Blum and R.~Koch.
\newblock Greibach normal form transformation revisited.
\newblock {\em Inf. Comput.}, 150:112--118, 1999.

\bibitem{copwin90}
D.~Coppersmith and S.~Winograd.
\newblock Matrix multiplication via arithmetic progressions.
\newblock {\em J. Symbolic Comput.}, 9:251--280, 1990.

\bibitem{hopull79}
J.~W. Hopcroft and J.~D. Ullman.
\newblock {\em Introduction to automata theory, languages, and computation}.
\newblock Addison-Wesley, 1979.

\bibitem{KS08}
S.~Kadioglu and M.~Sellmann.
\newblock Efficient context-free grammar constraints.
\newblock In {\em AAAI}, pages 310--316, 2008.

\bibitem{katnarwal08}
G.~Katsirelos, N.~Narodytska, and T.~Walsh.
\newblock The weighted {CFG} constraint.
\newblock In {\em CPAIOR}, pages 323--327, 2008.

\bibitem{gkatsiCPAIOR09}
G.~Katsirelos, N.~Narodytska, and T.~Walsh.
\newblock Reformulating global grammar constraints.
\newblock In {\em CPAIOR09}, pages 132--147, 2009.

\bibitem{lagerkvistthesis}
M.~Lagerkvist.
\newblock {\em Techniques for Efficient Constraint Propagation}.
\newblock PhD thesis, KTH, Sweden, 2008.

\bibitem{lee02}
L.~Lee.
\newblock Fast context-free grammar parsing requires fast boolean matrix
  multiplication.
\newblock {\em J. ACM}, 49:1--15, 2002.

\bibitem{pesant1}
G.~Pesant.
\newblock A regular language membership constraint for finite sequences of
  variables.
\newblock In {\em CP}, pages 482--495, 2004.

\bibitem{QuimperAAAI08}
C.~Quimper and T.~Walsh.
\newblock Decompositions of grammar constraints.
\newblock In {\em AAAI}, pages 1567--1570, 2008.

\bibitem{qwcp06}
C.~G. Quimper and T.~Walsh.
\newblock Global grammar constraints.
\newblock In {\em CP}, pages 751--755, 2006.

\bibitem{qwcp07}
C.~G. Quimper and T.~Walsh.
\newblock Decomposing global grammar constraints.
\newblock In {\em CP}, pages 590--604, 2007.

\bibitem{flhandbook}
G.~Rozenberg and A.~Salomaa.
\newblock {\em Handbook of Formal Languages}, volume~1.
\newblock Springer, 2004.

\bibitem{grammar2}
M.~Sellmann.
\newblock The theory of grammar constraints.
\newblock In {\em CP}, pages 530--544, 2006.

\bibitem{wagwec86}
K.~Wagner and G.~Wechsung.
\newblock {\em Computational Complexity}.
\newblock Springer, 1986.

\end{thebibliography}

\end{document}